\documentclass{article}

\usepackage[english]{babel}
\usepackage[toc, titletoc]{appendix}

\usepackage[letterpaper,top=2cm,bottom=2cm,left=2cm,right=2cm,marginparwidth=1.75cm]{geometry}
\usepackage{setspace}
\usepackage{lineno}

\usepackage{amsmath, amsfonts, amsthm, amssymb}
\usepackage{graphicx, float, booktabs, enumerate, enumitem, pgf}
\usepackage[colorlinks=true, allcolors=blue]{hyperref}
\usepackage{subcaption}
\usepackage{csquotes}
\usepackage{ulem}
\usepackage{fancybox}
\usepackage{algorithm, algpseudocode}
\usepackage{cases}
\usepackage[overload]{empheq}
\usepackage{cleveref} 

\usepackage{tikz}
\usetikzlibrary{trees}
\tikzstyle{level 1}=[level distance=3.5cm, sibling distance=3.5cm]
\tikzstyle{level 2}=[level distance=3.5cm, sibling distance=2cm]
\tikzstyle{bag} = [circle]
\tikzstyle{end} = [circle, minimum width=3pt,fill, inner sep=0pt]
\tikzstyle{mynode}=[thick,draw=black,circle, size=30]

\usetikzlibrary{backgrounds,arrows,shapes,tikzmark,calc}
\usetikzlibrary{tikzmark, arrows.meta}

\theoremstyle{plain}
\newtheorem{theorem}{Theorem}[section]

\newtheorem{corollary}[theorem]{Corollary}
\newtheorem{prop}[theorem]{Proposition}
\newtheorem{lemma}[theorem]{Lemma}
\theoremstyle{definition}
\newtheorem{definition}[theorem]{Definition}
\newtheorem{examplex}{Example}
\newenvironment{exm}
  {\pushQED{\qed}\examplex}
  {\popQED\endexamplex}

\usepackage{listings}
\usepackage{xcolor}
\definecolor{codegreen}{rgb}{0,0.6,0}
\definecolor{codegray}{rgb}{0.5,0.5,0.5}
\definecolor{codepurple}{rgb}{0.58,0,0.82}
\definecolor{backcolour}{rgb}{0.95,0.95,0.92}
\lstdefinestyle{mystyle}{
    backgroundcolor=\color{backcolour},   
    commentstyle=\color{codegreen},
    keywordstyle=\color{magenta},
    numberstyle=\tiny\color{codegray},
    stringstyle=\color{codepurple},
    basicstyle=\ttfamily\footnotesize,
    breakatwhitespace=false,         
    breaklines=true,                 
    captionpos=b,                    
    keepspaces=true,                 
    numbers=left,                    
    numbersep=5pt,                  
    showspaces=false,                
    showstringspaces=false,
    showtabs=false,                  
    tabsize=2
}
\usepackage{ulem}

\usepackage[style=numeric-comp,backend=biber]{biblatex}
\usepackage{authblk}
\title{Null Space Properties of Neural Networks with \\Applications to Image Steganography\\  }
\author[1]{Xiang Li}
\author[1]{Kevin M. Short}
\affil[1]{Integrated Applied Mathematics Program, Department of Mathematics and Statistics\\ University of New Hampshire, Durham, New Hampshire 03824, USA}

\date{}

\begin{document}
\maketitle

\begin{abstract}
This paper explores the null space properties of neural networks. 
We extend the null space definition from linear to nonlinear maps and discuss the presence of a null space in neural networks. 
The null space of a given neural network can tell us the part of the input data that makes no contribution to the final prediction so that we can use it to trick the neural network. 
This reveals an inherent weakness in neural networks that can be exploited.
One application described here leads to a method of image steganography. 
Through experiments on image datasets such as MNIST, we show that we can use null space components to force the neural network to choose a selected hidden image class, even though the overall image can be made to look like a completely different image. 
We conclude by showing comparisons between what a human viewer would see, and the part of the image that the neural network is actually using to make predictions and, hence, show that what the neural network ``sees'' is completely different than what we would expect.
\end{abstract}

\section{Introduction}
Neural networks are powerful learning methods in use for various tasks today.
This is especially true in the domain of image recognition, where neural networks can achieve even human-competitive results\cite{su2019one}.
However, a number of studies have revealed that neural networks for image classification can be easily influenced to misclassify by modifying images\cite{akhtar2018threat}.

In 2014, Szegedy et al. first discovered an intriguing weakness of deep neural networks\cite{szegedy2014intriguing}. 
They showed that neural networks for image classification can be easily fooled by small perturbations, 
and they called these intentionally modified images \textit{adversarial examples}. 
Following this observation, numerous studies have been carried out to find different ways to generate adversarial examples\cite{goodfellow2015explaining, moosavi2017universal,su2019one}.  
The main idea is to find a subtle perturbation that can drastically change the output of a neural network by adding it to the data. 
It is observed that adversarial examples have good transferability across models, which suggests that the existence of adversarial examples is also a property of datasets\cite{ilyas2019adversarial},
thus adversarial examples are not restricted only to the given model. 
In our study, we aim to find a model-based method to fool the neural networks. 
In a manner different from the adversarial examples, we take advantage of certain null space properties of the neural network, and generate examples to fool the neural network by adding a large difference to images without changing the predictions. 
More importantly, to the naked eye, the large differences will look like completely different objects than those that will be recognized by the neural network, in a form of image steganography.

A null space is an important concept defined for linear transformations. 
Since neural networks are nonlinear maps, few studies have focused on the null space of neural networks. 
Cook et al. integrated the null space analysis on weight matrices with the loss function, and proposed an outlier detection method directly into a neural network for classification tasks\cite{cook2020outlier}.
Rezaei et al. analyzed the null space of the last layer weight matrix of neural networks, and used it to quantify overfitting without access to training data or knowledge of the accuracy of those data\cite{rezaei2023quantifying}.
However, there is still a lack of understanding of the global effects of null space properties for neural networks. 
In this study, 
inspired by the null space of linear maps,
we introduce the concept of a null space for nonlinear maps, and discuss how the null space applies to neural networks. 
We show how the null space is generally an intrinsic property of neural networks. 
Further, once the neural network's architecture is determined, the dimension of the null space of the neural network is also determined in most cases. 
To illustrate this in concrete terms, we use null space-based methods to fool an image recognition neural network as an application of image steganography.

Image steganography is a technique to hide an image inside another image\cite{subramanian2021image}.
In addition to traditional-based steganography methods,
neural networks are also widely used for image steganography.
Neural networks have been employed for image steganography using various different approaches \cite{wu2018stegDCN, duan2019stegUNet, baluja2019stego3part}. 
Though conceptually similar to other steganographic methods, the null space-based method presented here gives new degrees of freedom and a clear process for creating steganographic images for neural nets.
The main goal of this method is to hide an image and recognize the correct class of the hidden image, while presenting to the (human) viewer an image that is completely different than the hidden image that will be recognized by the neural network.

This paper is organized as follows: 
Section \ref{sect:methodology} introduces the null space of neural networks. 
Inspired by the null space analysis, we propose an image steganography method based on the null space of fully connected neural networks in Section \ref{sect:imagestego}. 
In Section \ref{sect:results}, we perform a number of experiments on different image datasets. 
In Section \ref{sect:conclusion}, we offer some discussion and some conclusions about the advantages and limitations of the null space method, and the differences between the null space image steganography method and adversarial examples. 
We also present some images showing that what the NN is seeing is not what we humans might think it is seeing.

\section{Methodology}\label{sect:methodology}
In this section, we introduce the null space of nonlinear maps and discuss the null space of neural networks. 
As an application of the null space for neural networks, we propose a new method for image steganography.

\subsection{Null space of linear and nonlinear maps}\label{sect:nullspace}
In linear algebra, the null space of an $m\times n$ matrix $A$ is $\text{Null}(A)=\{\vec{x}\in\mathbb{R}^n: A\vec{x}=\vec{0}\}$;
$\text{Null}(A)$ is a subspace of $\mathbb{R}^n$. 
Consider a linear map $T:\mathbb{R}^n\rightarrow\mathbb{R}^m, T(\vec{x})=A\vec{x}$ with $\text{Null}(A)$ nontrivial. By the definition of a null space, it is not difficult to see that for any $\vec{x}\in\mathbb{R}^n$ and $\vec{x}_{null}\in\text{Null}(A)$, $T(\vec{x}+\vec{x}_{null})=T\vec{x}$. 
That is, adding any vector in $\text{Null}(A)$ to any input $\vec{x}$ won't change the output since the transformation is a linear map. 

In most cases, nonlinear maps do not have the definition of null space as described above. 
The set $\{\vec{x}: f(\vec{x})=0 \}$ for a nonlinear map $f:\mathbb{R}^n\rightarrow\mathbb{R}^m$ is not a subspace of $\mathbb{R}^n$. 
Moreover, the property of linearity $f(\vec{x}+\vec{y})=f(\vec{x})+f(\vec{y})$ does not hold for the nonlinear map $f$.
Although we can still find the set $\{\vec{x}_\alpha: f(\vec{x}_\alpha)=0 \}$, the property $f(\vec{x}+\vec{x}_{\alpha})=f(\vec{x})$ is no longer true for all $\vec{x}\in\mathbb{R}^n$. 
To preserve this useful property, the concept of the null space can be adapted to nonlinear maps by identifying a set of vectors $\vec{x}_{null}$ that satisfy the condition $f(\vec{x}+a\vec{x}_{null})=f(\vec{x})$  for any $\vec{x}\in\mathbb{R}^n, a\in\mathbb{R}$:

\begin{definition}[Null space of a nonlinear map]
The \textbf{null space of a nonlinear map} $f: \mathbb{R}^n\rightarrow \mathbb{R}^m$, denoted as $N(f)$, is the set of all vectors that by adding these vectors to any input, the image under the map, or output, will not change:
\begin{equation}\label{eqn:NSdef}
    N(f):=\{ \vec v\in \mathbb{R}^n: f(\vec x) = f(\vec x+a\vec v) \text{ for all } \vec x\in \mathbb{R}^n, a\in\mathbb{R} \}.
\end{equation}
\end{definition}

When this definition is applied to linear maps, it is equivalent to the null space concept in linear algebra. 
However, for the purposes of this paper, we will use the term ``null space'' in a broader sense, extending its application to include affine maps, nonlinear maps, etc.
Similar to the null space in linear algebra, $N(f)$ is a subspace of $\mathbb{R}^n$ (Proposition \ref{prop:n-subsp}). 
Furthermore, as demonstrated in Proposition \ref{prop:same-image}, all the vectors in $N(f)$ are mapped to a single point by $f$.
More specifically, for any $\vec{x}, \vec{y}\in N(f),\; f(\vec{x})=f(\vec{y})$.
Further properties about null spaces of nonlinear maps are provided in Appendix \ref{appNS}.
One difference between null spaces of linear and nonlinear maps is that a nonlinear map will not always map null space vectors to zeros.

The definition of the null space for nonlinear maps is straightforward. 
However, in practice, it is difficult to find a null space explicitly by following this definition. 
Next, we will introduce an alternative, yet equivalent definition for the null space of nonlinear maps, which provides a more accessible way to find a null space of a nonlinear map.

To begin with, given a nonlinear map $f: \mathbb{R}^n\rightarrow \mathbb{R}^m$, let's consider the subspaces of $N(f)$ first. 
Suppose $f$ has a decomposition, that is, $f$ can be expressed as the composition of two functions $f=f_2\circ f_1$, and $f_1$ is linear. 
It follows that $\text{Null}(f_1)$ is a subspace of $N(f)$. 
This is straightforward to see, 
since for any $\vec{x}_{null}\in\text{Null}(f_1)$, we have that, $f_1(\vec{x})=f_1(\vec{x}+\vec{x}_{null})$ holds for every $\vec{x}\in\mathbb{R}^n$, 
and so $f(\vec{x}+\vec{x}_{null})=f_2\circ f_1(\vec{x}+\vec{x}_{null})=f_2\circ f_1(\vec{x})=f(\vec{x})$.
Therefore, now we have a way to find a subspace of $N(f)$.
Hence, we define a \textbf{partial null space}. 
\begin{definition}[Partial null space of a map]
Given a nonlinear map $f: \mathbb{R}^n\rightarrow \mathbb{R}^m$. If $f$ has a decomposition $f= f_2\circ f_1$, where $f_1:\mathbb{R}^n \rightarrow \mathbb{R}^d$ is a linear map,
then the null space of $f_1$ is defined as a partial null space of $f$ (given by $f_1$). We denote it as $PN_{f_1}(f)$.
\end{definition}

Then, a natural question is how to find the null space with partial null spaces and the decomposition of maps. 
In the following lemmas, we present some results between partial null spaces and the null space of a nonlinear map. Detailed proofs are provided in Appendix \ref{appNS}.

\begin{lemma}\label{lem:pn_subspace}
Let $f:\mathbb{R}^n\rightarrow \mathbb{R}^m$ be a nonlinear map. 
Every partial null space $PN(f)$ is a subspace of $N(f)$, $\dim PN(f)\leq \dim N(f)$.
\end{lemma}

\begin{lemma}\label{lem:pn}
For every nonlinear map $f: \mathbb{R}^n\rightarrow \mathbb{R}^m$,
there exists a decomposition $f= f_2\circ f_1$ with $f_1$ being a linear map such that $PN_{f_1}(f)=\text{Null}(f_1)= N(f)$, i.e., the partial null space given by $f_1$ is equal to $N(f)$.
\end{lemma}

Note that $PN_{f_1}(f)=\text{Null}(f_1)= N(f)$ is not always true for any arbitrary decomposition $f=f_2\circ f_1$. In the following corollary, we further discuss the conditions under which this equation holds.

\begin{corollary}\label{cor:NSdef2}
Let $f:\mathbb{R}^n\rightarrow \mathbb{R}^m$ be a nonlinear map. 
The null space of $f$ is the largest partial null space of $f$. That is, $N(f)=PN_{f_1}(f)$ if $f=f_2\circ f_1$, $f_1$ is a linear map and $PN_{f_1}(f)$ has the largest dimension among all decompositions of $f$.
\end{corollary}

Directly from Lemma \ref{lem:pn_subspace} and \ref{lem:pn}, Corollary \ref{cor:NSdef2} gives an equivalent definition of $N(f)$ in terms of its partial null space. 
Consequently, instead of searching for all possible vectors that satisfy equation (\ref{eqn:NSdef}), we have turned the null space problem into the problem of finding a decomposition of $f$ that yields the largest partial null space.

\subsection{Null space of fully connected neural networks}\label{FCNN}
\begin{definition}
Consider a fully connected neural network $f: \mathbb{R}^{n_0}\rightarrow\mathbb{R}^{n_{K+1}}$, 
where $K\in\mathbb{N}$ denotes the \textbf{number of hidden layers} of $f$, with $n_1, n_2, \dots, n_K\in\mathbb{N}$ representing \textbf{widths of the hidden layers}. $n_0, n_{K+1}\in\mathbb{N}$ are input and output dimensions, respectively. 
$f$ is called a $(K + 1)-$layer Fully Connected Neural Network (FCNN). 
In this paper, we refer to this FCNN architecture as a $(n_0, n_1, \dots, n_K)-$FCNN.

The network $f$ consists of $K+1$ paired linear and affine transformations, each defined as $T_i(\vec{x})=W_i \vec{x}, A_i(\vec{x}) = \vec{x} +\vec{b}_i$, respectively. 
Here, $W_i$ represents the weight matrix and $\vec{b}_i$ represents the bias vector in the $i^{th}$ layer, for $i=1,\dots, K+1$.
The activation functions are denoted by $\sigma$. 
Common activation functions include sigmoid activation function and Rectified Linear Unit (ReLU).

The function $f$ then can be represented by a composition of maps:
\begin{equation}\label{eqn:NN}
    f = A_{K+1} \circ T_{K+1} \circ \sigma \circ T_K \circ \cdots \circ A_2 \circ T_2 \circ \sigma \circ A_1 \circ T_1
\end{equation}

\end{definition}

For more introduction to the neural networks, see Appendix \ref{appIntro}. 

According to Corollary \ref{cor:NSdef2}, we can do null space analysis on FCNNs using the composite function form (\ref{eqn:NN}). 
An FCNN can be naturally decomposed into two parts, $(A_{K+1} \circ T_{K+1} \circ \sigma \circ T_K \circ \cdots \circ A_2 \circ T_2 \circ \sigma \circ A_1)$ and $T_1$, where $T_1$ is linear. 
Thus, $\text{Null}(T_1)$ is a partial null space of FCNN $f$ which, in most cases, is also the null space $N(f)$. Here is an example.

\begin{exm}[ReLU neural networks (ReLU NNs) and null space]
A ReLU NN is an FCNN with activation function $\sigma_R(x)=\max(x, 0)$. 
Consider a $(n_0, n_1, \dots, n_K)-$ReLU NN defined as $f = A_{K+1} \circ T_{K+1} \circ \sigma_R \circ T_K \circ \cdots \circ A_2 \circ T_2 \circ \sigma_R \circ A_1 \circ T_1$, where $A_i$ and $T_i$ denote affine transformations and linear maps, respectively. 
$f$ can be decomposed into two maps $f=(A_{K+1} \circ T_{K+1} \circ \sigma_R \circ T_K \circ \cdots \circ A_2 \circ T_2 \circ \sigma_R \circ A_1) \circ T_1$.
Thus, we know that a partial null space of $f$ is $PN_{T_1}(f)=\text{Null}(T_1)$. 
In most cases, this is also the null space of $f$, $N(f)=PN_{T_1}(f)$.
Further, assume the first hidden layer weight matrix is $W_1$, $N(f)=\text{Null}(W_1)$.

To visualize the null space of a ReLU NN, Figure \ref{fig:reluNN} shows the surface of a $(2,1,3,1)-$ReLU NN combined with a contour plot beneath the surface plot for clearer interpretation. 
The figure shows that this ReLU NN has a one-dimensional null space, represented by the parallel lines running back and into the page on the contour plot. 
Notably, along the null space direction, all points on the surface are marked with the same color, which means they have the same output value. 
This implies that regardless of the distance traversed in the null space direction, the output of the ReLU NN remains constant. 
Similarly, in higher-dimensional cases, a ReLU NN with a higher-dimensional null space would also have the same values along the high-dimensional null space plane or hyperplane.

\begin{figure}[H]
\centering
\begin{tikzpicture}

\node[inner sep=0pt] (image) at (0,0) {\includegraphics[width=0.55\textwidth]{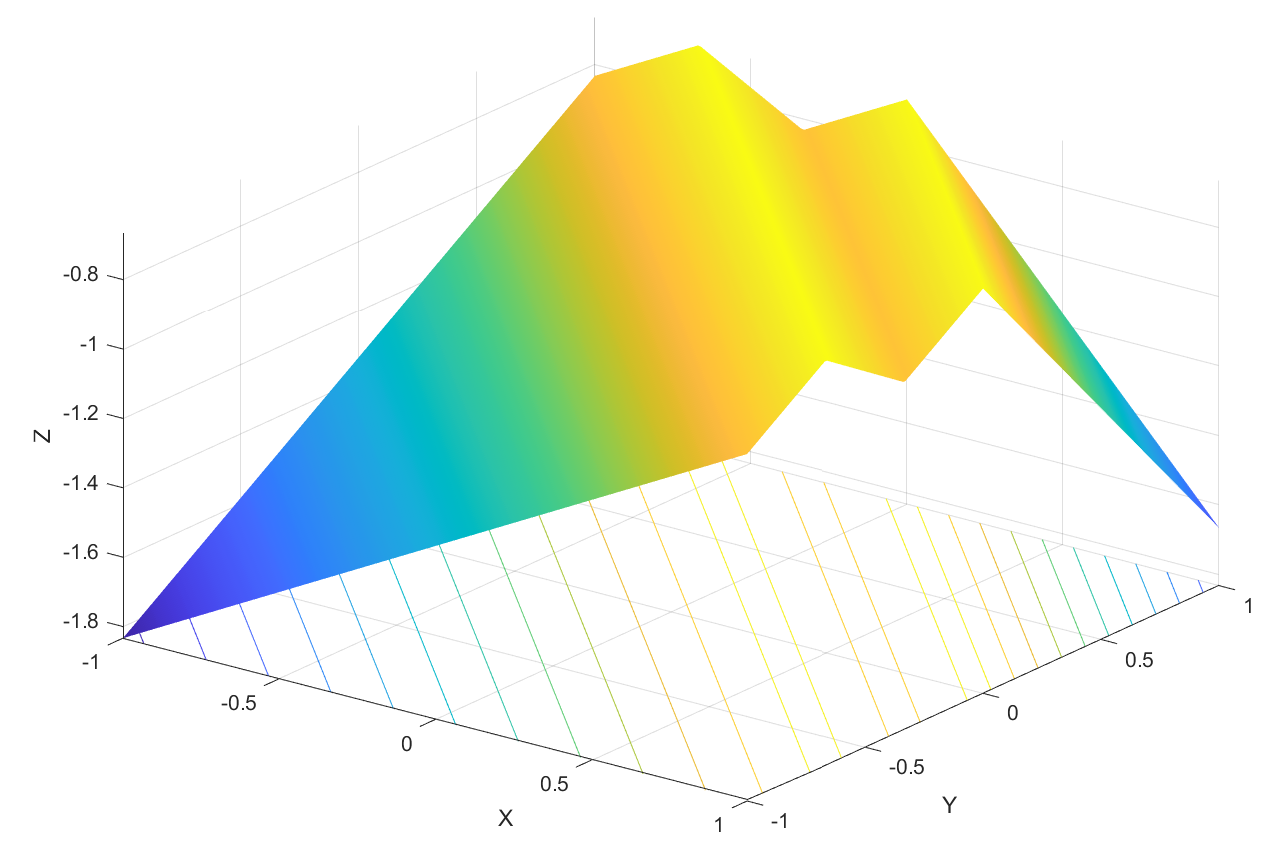}};
    
\draw[->, red] (1.5, -1.5) -- ++(-0.2,0.6);
\node[anchor=west] at (1.,-1.75) {\tiny \textcolor{red}{Null space direction $\vec{v}$}};
\end{tikzpicture}

\caption{The surface plot of a (2,1,3,1)-ReLU NN.}
\label{fig:reluNN}
\end{figure}
\end{exm}


\subsection{Null space of convolutional neural network}\label{CNN}
Similarly, a convolutional neural network (CNN) can also be represented by the composition of maps. 
Instead of having only affine transformations and activation functions, the early layers of a CNN allow two additional types of computation: convolution, denoted as $C$, and pooling, denoted as $S$.
To determine the null space of a CNN $f$, we can first decompose $f$ into the composition of maps $f=T_{k+1}\circ\sigma\circ\dots\circ\sigma\circ C$. 
So, in most cases, the first convolution operation $C$ is the maximum linear map of $f$, and $N(f)=\text{Null}(C)$.

The null space of a CNN is significantly more complex than an FCNN. 
In this paper, we will only give an overview of the null space analysis for CNN in a simple case. 
Consider the case when the input image of a CNN only has one channel, for example, the MNIST dataset. 

A typical convolutional layer in a CNN comprises multiple kernels. 
To find the null space of the entire convolutional layer $\text{Null}(C)$, we can initially determine the null space of a single kernel $\text{Null}(K)$ corresponding to one convolution operation.
As shown in Appendix \ref{appCNN}, given one kernel, if the convolution operation keeps the output image with the same dimensions as the input image (same padding), then $N(f)=\{0\}$ in most cases. 
If the output image of the first convolution operation has fewer dimensions than the input image (valid padding), $N(f)$ can be nontrivial.
According to the Lemma \ref{lem:conv}, under most cases, the null space of a given kernel has a dimension that is equal to the difference between the dimensions of input and output image. 
For example, given a $28\times 28$ image and a $3\times3$ kernel, the dimension of the null space of this convolution operation is $28\times28-26\times26=108$ in most cases.

Assume the first convolutional layer $C$ has $n$ kernels $K_1, \cdots, K_n$, then $N(f)=N(C)=\bigcap\limits_{i=1}^n \text{Null}(K_i)$.
For instance, if the input is a $28\times28$ image, and the first convolutional layer has six $3\times 3$ kernels, then the null space of the CNN is the intersection of six $109-$dimensional subspaces of $\mathbb{R}^{784}$. 
The likelihood of there being a non-trivial intersection between 109-dimensional subspaces of a 784-dimensional space is low, but care must be taken when designing the kernels. 
Compared to FCNNs, CNNs can achieve a trivial null space with fewer unknown parameters in the first weight matrix and perhaps greater robustness.

\section{Photo steganography based on the null space of a neural network}\label{sect:imagestego}
In considering the task of hiding secret information inside another image, a natural idea emerges from our previous discussions on the null space properties of FCNNs.

Let us denote by $f$ an FCNN designed and trained for image classification.
For an image $X$, $f(x)$ predicts the class of image $X$. 
Assume the null space of this FCNN $f$ is non-trivial (i.e., $N(f)\neq \{ \mathbf{0} \}$).
Using the null space of $f$, any image $X$ can be decomposed into two components: the orthogonal projection onto the null space $\hat{X}$ and its orthogonal complement $X_{\perp}$. 
Provided that the dimension of $N(f)$ is sufficiently large, $\hat{X}\in N(f)$ can retain most of the visual features while having no contribution to the final prediction (with $f$); $X_{\perp}$ extracts all the important information for prediction with $f$, but it would have an entirely different visual appearance from the original image $X$.

Moreover, we can choose any $X_{null}\in N(f)$ to construct a modified image $\Tilde{X}=X+X_{null}$.
Visually, this modified image can be designed to appear significantly different from the original image $X$, while it preserves the original class prediction with $f$, that is, $f(\Tilde{X})=f(X)$. 
Further, if we form combinations using the orthogonal complement, such as if
$Y = Y_{null} + Y_{\perp}$, then a combination $Z = X_{null} + Y_\perp$ would be classified based only on $Y_\perp$.

Following the ideas above, we propose a method for crafting \textit{steganographic images}, referred to as ``stego images'' in this paper, to fool FCNNs and pass on secret classification information. 
Consider two images $H$ and $C$. 
$H$ is the \textit{hidden image} containing the information that we intend to transmit secretly. 
$C$ is the \textit{cover image} used to conceal the hidden information. 
Suppose our goal is to pass an image $S$ that looks like $C$ but secretly hides the classification information of $H$.
This can be achieved by decomposing each image into two components based on the null space $N(f)$, and combining the null space component of $C$ with the predictable component of $H$ (i.e., the orthogonal projection $H_\perp$). 
Consequently, the resulting stego image $S$ will visually be close to the appearance of $C$ while being classified under the same category as $H$ by the neural network $f$.


Let us use the MNIST dataset to provide a more specific description of the entire image steganography algorithm. 
Given the MNIST dataset ${(X_i, y_i)}_{i=1}^N$, where $X_i$ are grayscale images with $28\times 28$ pixels, normalized to fall within the range $[-1,1]$; $y_i$ are labels. 
A pre-trained FCNN $f$ is needed beforehand. The neural network should be trained using both the original images $X_i$ and their rescaled images $\alpha X_i$, $\alpha\in(0,1)$. 
This provides ``headroom'' when combining images, but has no significant impact on the accuracy of the predictions. 
The initial step is to find the null space $N(f)$ of FCNN, which can be done by applying singular value decomposition to the first hidden layer weight matrix $W_1=USV^T$ and the columns of $V$ corresponding to zero singular values form a basis of $N(f)$. 
Then, given the cover image $C$ and hidden image $H$, using the null space $N(f)$, create a steganographic image $S$, which looks like the cover image $C$ but has the same classification as the hidden image $H$. 
Specifically, to do this, we find the orthogonal projection $\hat{C}$ of $C$ onto $N(f)$ and orthogonal complement $H_{\perp}$ of $H$.
Then, we create the stego image $S$ with the linear combination $\alpha_1H_{\perp}+\alpha_2\hat{C},\; \alpha_1,\alpha_2\in(0, 1)$ such that $S\in [-1,1]^n$. The aim of this step is to guarantee that $S$ is in the domain of $f$ and $f(S)=f(\alpha_2 \hat{C})=f(\alpha_2 C)$ can pass the correct prediction for the hidden image. 
A visualization of the whole process is given in figure \ref{fig:stego-steps}. 
This algorithm is also applicable to other datasets. In the next section, we will show the experiments with MNIST, the Fashion-MNIST (FMNIST), and Extended MNIST (EMNIST) datasets. 
For simplification, we will only present experiment results with $\alpha_1=0.2$.

\begin{figure}[H]
    \centering
    \includegraphics[width=.7\textwidth]{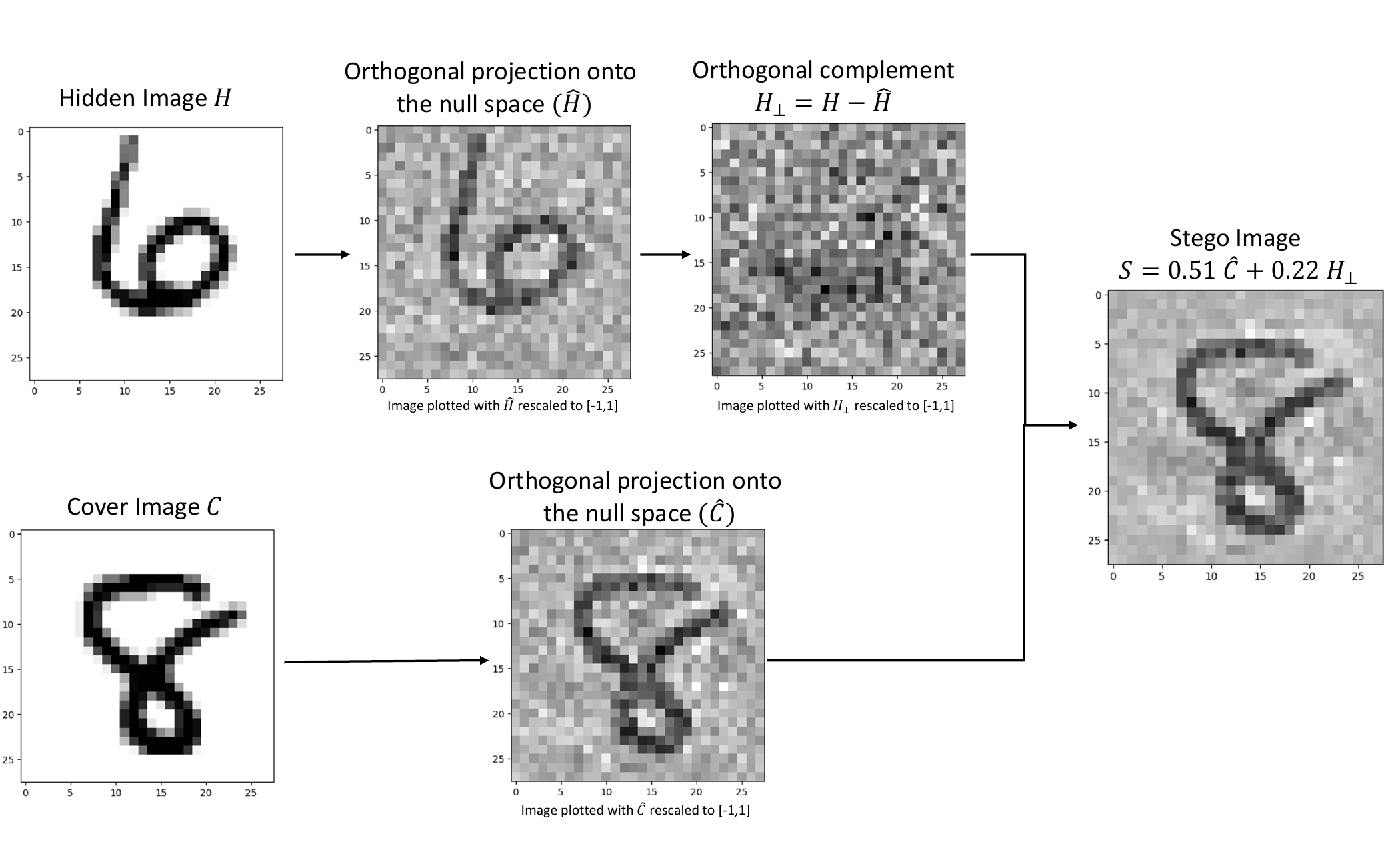}
    \caption{An example of creating a steganographic image with null space of a ReLU NN.}
    \label{fig:stego-steps}
\end{figure}

For the example in Figure \ref{fig:stego-steps}, and any other examples created with the algorithm, the NN will classify the image as the category of the hidden image and will completely ignore the cover image.

\section{Experimental results}\label{sect:results}
All experiments presented in this section were implemented in TensorFlow. 
The datasets employed for these experiments include MNIST\cite{MNIST}, Fahion-MNIST\cite{FMNIST}, Extended MNIST (EMNIST)\cite{EMNIST}, and CIFAR-10\cite{CIFAR10}.
\subsection{Hide the digits}
In this section, we use the MNIST dataset to conduct experiments in image steganography. 
For these experiments, we trained a $(784, 32, 16, 10)-$ ReLU NN based on an expanded image training data set that included the original and rescaled with $\alpha_1= 0.2$ datasets with the prediction accuracy of $99.79\%$ on the original training data. 
For the original training data it correctly predicted, the average confidence level is $99.98\%$. 
When evaluated on the rescaled training data, this ReLU NN maintained a high prediction accuracy of $99.78\%$ with an average confidence of $99.77\%$ on correctly predicted data. 
Notably, this ReLU NN has a 752-dimensional null space. 
This large null space plays a crucial role in the steganographic capabilities of the network.

We then filter out only the correctly predicted images to create a new dataset. This dataset comprises a total of 50000 images, with each class having 5,000 examples.
The $(784, 32, 16, 10)-$ReLU NN can predict correctly on both original and rescaled data from this new dataset. 
The main purpose of this process is to ensure that for any chosen hidden image $H$ from the new dataset, feeding $0.2H$ into the ReLU NN does not create scaling issues, and yields a correct classification. 
Next, we will show some representative results with stego images produced by $S=0.2H_{\perp}+\alpha_2\hat{C}$.

Figure \ref{fig:stego-ex} and \ref{fig:MNIST-more} present stego images generated using the images from the new dataset. 
For each set of three images, the first is the cover image, the second is the hidden image, and the third, the stego image, combines parts of the cover and the hidden image. 
Additionally, Figure \ref{fig:stego-ex} annotates the combination weights $\alpha_1,\alpha_2$ of stego images, as well as the predictions and their corresponding confidence levels. 
In each case, the ReLU NN predicts the hidden digit with high confidence.


\begin{figure}[H]
\centering
\begin{subfigure}{.48\textwidth}
  \centering
  \includegraphics[width=.98\linewidth]{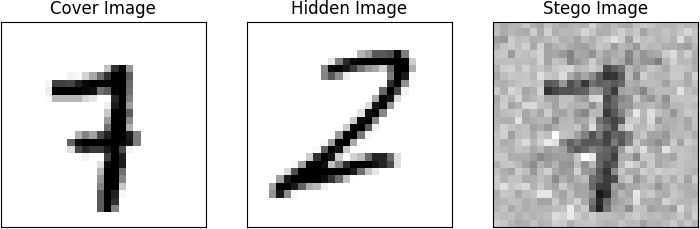}
  \caption{The cover image is predicted as 7 with a confidence of nearly 100\%. The hidden image is predicted as 2 with a confidence of nearly 100\%. \textbf{The stego image}, $S=0.2H_{\perp}+0.51\hat{C}$, \textbf{is predicted as 2} with a confidence level 99.9\%.}
  \label{fig:MNIST-ex1}
\end{subfigure}%
\hfill
\begin{subfigure}{.48\textwidth}
  \centering
  \includegraphics[width=.98\linewidth]{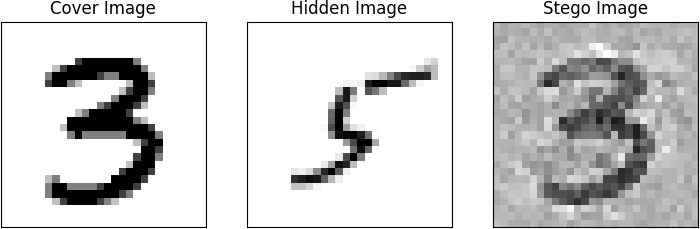}
  \caption{The cover image is predicted as 3 with a confidence of nearly 100\%. The hidden image is predicted as 5 with a confidence of nearly 100\%. \textbf{The stego image}, $S=0.2H_{\perp}+0.49\hat{C}$, \textbf{is predicted as 5} with a confidence level 99.9\%.}
  \label{fig:MNIST-ex2}
\end{subfigure}
\caption{Examples of steganographic images with MNIST dataset.}
\label{fig:stego-ex}
\end{figure}

\begin{figure}[H]
    \centering
    \begin{subfigure}{0.22\textwidth}
        \includegraphics[width=\linewidth]{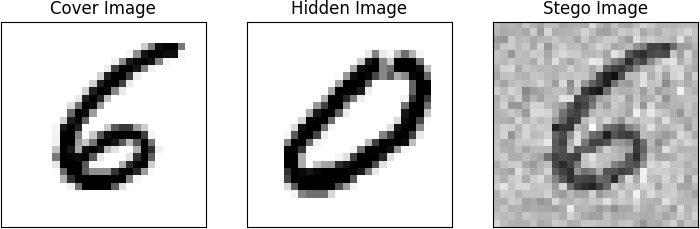}
    \end{subfigure}
    \hfill
    \begin{subfigure}{0.22\textwidth}
        \includegraphics[width=\linewidth]{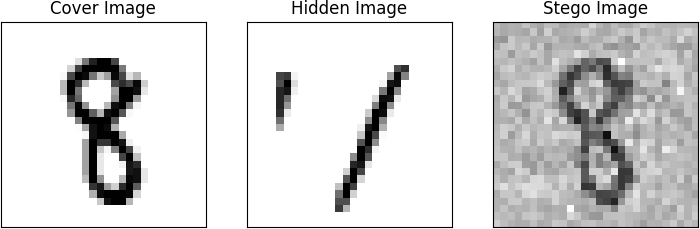}
    \end{subfigure}
    \hfill
    \begin{subfigure}{0.22\textwidth}
        \includegraphics[width=\linewidth]{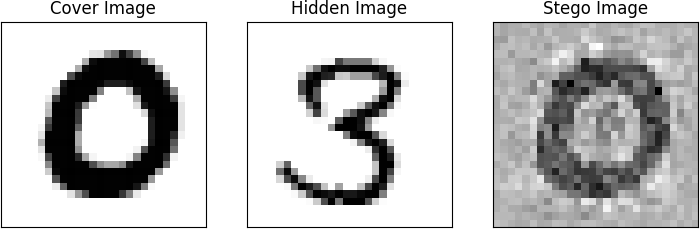}
    \end{subfigure}
    \hfill
    \begin{subfigure}{0.22\textwidth}
        \includegraphics[width=\linewidth]{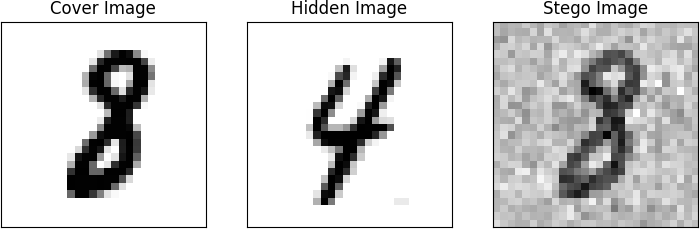}
    \end{subfigure}

    \medskip

    \begin{subfigure}{0.22\textwidth}
        \includegraphics[width=\linewidth]{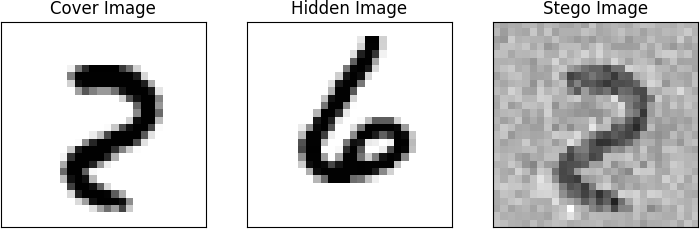}
    \end{subfigure}
    \hfill
    \begin{subfigure}{0.22\textwidth}
        \includegraphics[width=\linewidth]{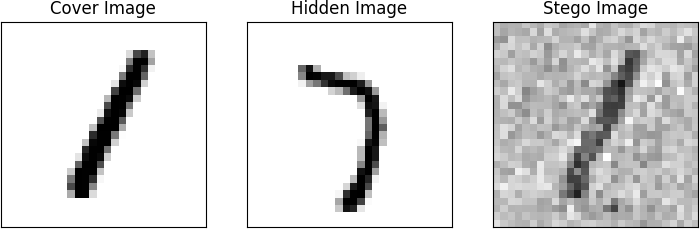}
    \end{subfigure}
    \hfill
    \begin{subfigure}{0.22\textwidth}
        \includegraphics[width=\linewidth]{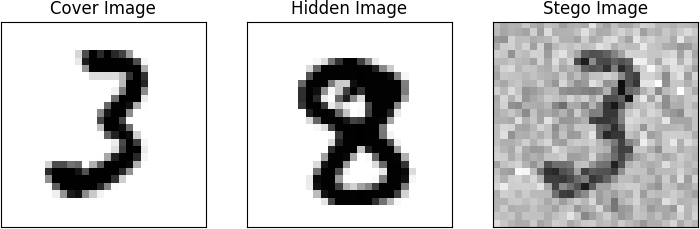}
    \end{subfigure}
    \hfill
    \begin{subfigure}{0.22\textwidth}
        \includegraphics[width=\linewidth]{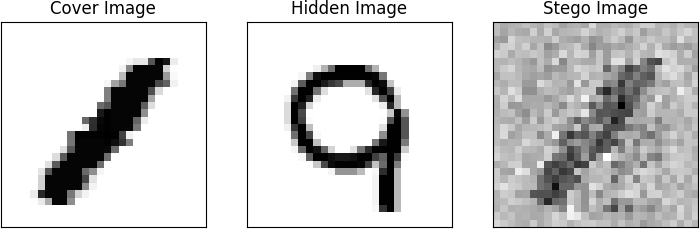}
    \end{subfigure}
    
    \medskip

    \begin{subfigure}{0.22\textwidth}
        \includegraphics[width=\linewidth]{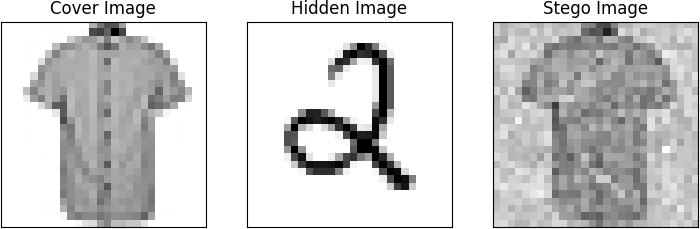}
    \end{subfigure}
    \hfill
    \begin{subfigure}{0.22\textwidth}
        \includegraphics[width=\linewidth]{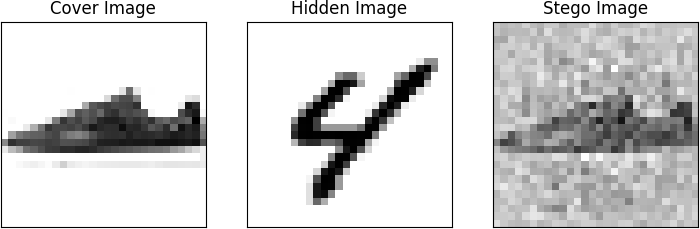}
    \end{subfigure}
    \hfill
    \begin{subfigure}{0.22\textwidth}
        \includegraphics[width=\linewidth]{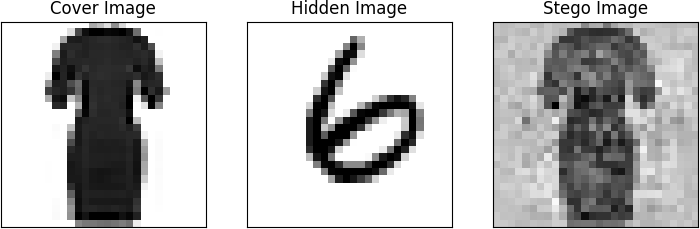}
    \end{subfigure}
    \hfill
    \begin{subfigure}{0.22\textwidth}
        \includegraphics[width=\linewidth]{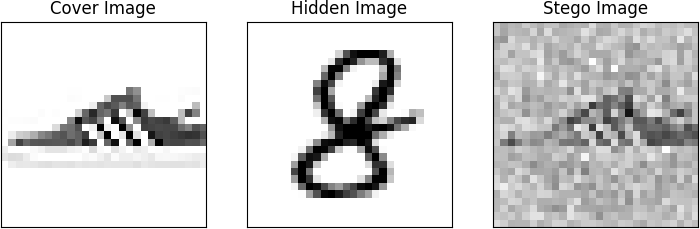}
    \end{subfigure}
    \caption{More examples on MNIST dataset.}
    \label{fig:MNIST-more}
\end{figure}

As shown in the figures, we can barely see the hidden digits in the stego images, which have patterns extremely similar to the cover images, even if the cover image is from another dataset.
However, these stego images are not predicted to be in the same classes as the cover images, rather, they are identified as the hidden images with high confidence. 
The cover image's portion within the stego image falls entirely in the null space of the ReLU NN, giving the stego image a visible pattern resembling the cover image (plus some noise-like structure) but makes no contribution to the prediction with the ReLU NN. 
On the other hand, the part of the hidden image included in the stego image appears like noise (as shown in figure \ref{fig:stego-steps}), but it is this portion that carries the most important information for prediction with ReLU NN. 

\subsection{Results on other datasets}
In addition to the MNIST dataset, We applied our method to another grayscale image dataset, FMNIST.

For the FMNIST dataset, we trained a $(784, 32, 16, 10)-$ ReLU NN with a 752-dimensional null space. 
To generate stego images, we first created a new dataset comprising both the original and rescaled images (also scaled by a factor of 0.2) that are predicted correctly. 
This new dataset has an average confidence of $99.74\%$ for original images and $97.22\%$ for the rescaled images. 
Examples of stego images created from the new dataset are shown in figure \ref{fig:stego-fmnist} and \ref{fig:FMNIST-more}. 

\begin{figure}[H]
\centering
\begin{subfigure}{.48\textwidth}
  \centering
  \includegraphics[width=.98\linewidth]{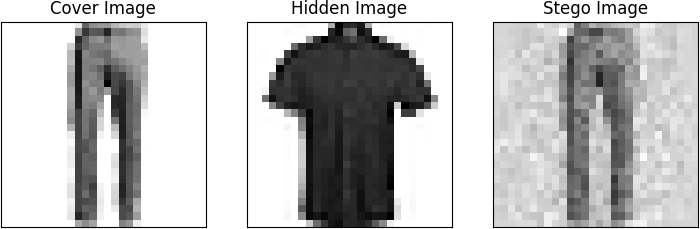}
  \caption{The cover image is predicted as ``Trouser'' with a confidence of nearly 100\%. The hidden image is predicted as ``Shirt'' with a confidence of nearly 100\%. The stego image, $S=0.2H_{\perp}+0.68\hat{C}$, is predicted as ``Shirt'' with a confidence level 86.4\%.}
  \label{fig:FMNIST-ex1}
\end{subfigure}%
\hfill
\begin{subfigure}{.48\textwidth}
  \centering
  \includegraphics[width=.98\linewidth]{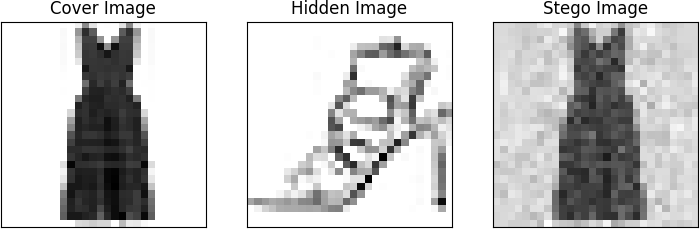}
  \caption{The cover image is predicted as ``Dress'' with a confidence of nearly 100\%. The hidden image is predicted as ``Sandal'' with a confidence of nearly 100\%. The stego image, $S=0.2H_{\perp}+0.49\hat{C}$, is predicted as ``Sandal'' with confidence of nearly 100\%.}
  \label{fig:FMNIST-ex2}
\end{subfigure}
\caption{Examples of steganographic images with FMNIST dataset.}
\label{fig:stego-fmnist}
\end{figure}

\begin{figure}[H]
    \centering
    \begin{subfigure}{0.22\textwidth}
        \includegraphics[width=\linewidth]{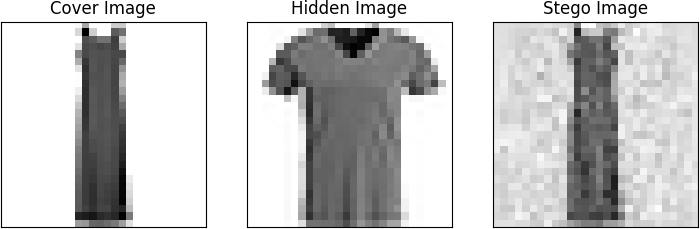}
    \end{subfigure}
    \hfill
    \begin{subfigure}{0.22\textwidth}
        \includegraphics[width=\linewidth]{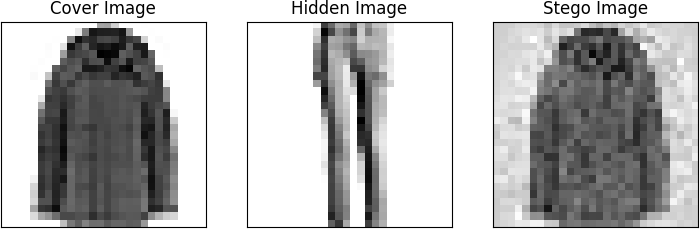}
    \end{subfigure}
    \hfill
    \begin{subfigure}{0.22\textwidth}
        \includegraphics[width=\linewidth]{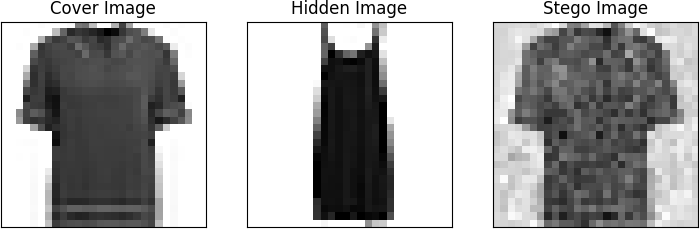}
    \end{subfigure}
    \hfill
    \begin{subfigure}{0.22\textwidth}
        \includegraphics[width=\linewidth]{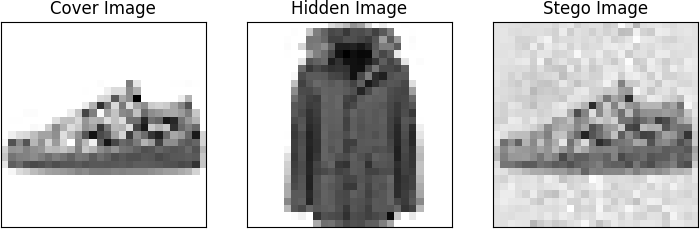}
    \end{subfigure}

    \medskip

    \begin{subfigure}{0.22\textwidth}
        \includegraphics[width=\linewidth]{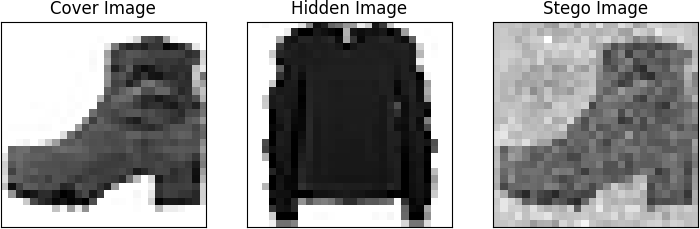}
    \end{subfigure}
    \hfill
    \begin{subfigure}{0.22\textwidth}
        \includegraphics[width=\linewidth]{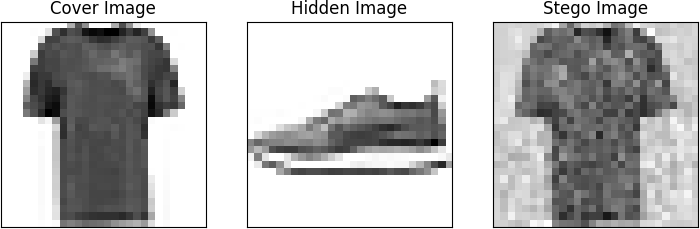}
    \end{subfigure}
    \hfill
    \begin{subfigure}{0.22\textwidth}
        \includegraphics[width=\linewidth]{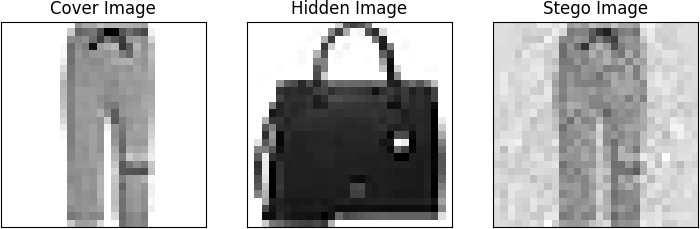}
    \end{subfigure}
    \hfill
    \begin{subfigure}{0.22\textwidth}
        \includegraphics[width=\linewidth]{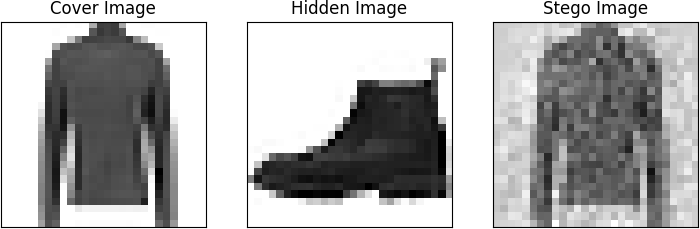}
    \end{subfigure}
    \caption{More examples on FMNIST dataset.}
    \label{fig:FMNIST-more}
\end{figure}

Similarly, as expected, the stego images have the look of cover images (as seen in the first column) but are recognized as the same categories as their corresponding hidden images (shown in the second column) with high confidence.

It is noteworthy that in Figure \ref{fig:FMNIST-ex1}, there is a decrease in confidence when predicting it as the hidden category. 
While the original hidden image $H$ is predicted as a ``shirt'' with confidence close to $1$, 
the stego image $S=0.2H_{\perp}+0.68\hat{C}$ is predicted as a ``shirt'' with lower confidence $86.4\%$. 
The reduced confidence level is caused by the rescaled hidden image. 
According to the null space method, the prediction and confidence level for the stego image $S$ should align with those for $0.2H_{\perp}$ and, consequently, $0.2H$.
In this case, if we input just the rescaled image to the NN, the rescaled hidden image $0.2H$ is also predicted as ``shirt'' with a confidence of $86.4\%$, which is consistent with our analysis. 
Therefore, in the null space method, the prediction and confidence are only related to the hidden image component of the stego image, i.e., if the stego image $S=\alpha_1H_{\perp}+\alpha_2\hat{C}$, as long as $\alpha_1H_{\perp}$ can be predicted correctly with high confidence, the stego image will also be identified as the hidden category with same confidence.

Unlike the experiment with the MNIST dataset, it is more common to see lower confidence stego image examples in the experiment with the FMNIST dataset.
In figure \ref{fig:stego-confi-ex}, we compare the original hidden image $H$ and rescaled hidden image $0.2H$ in Figure \ref{fig:FMNIST-ex1}.
Except for the prediction as ``shirt'' with the confidence of $86.4\%$, 
the rescaled image $0.2H$ is also predicted as ``T-shirt/Top'' with a confidence of $13.4\%$.
As shown in the figure \ref{fig:stego-confi-ex}, the rescaled image $0.2H$ has lower contrast. 
The rescaling operation seems to lead to the loss of some visual details, which causes lower confidence in the rescaled images. 
For instance, the buttons and the collar can not be clearly observed in the rescaled image, so it is also more possible to be identified as ``T-shirt/Top''. 

Therefore, the achievement of high confidence in steganographic images depends on having ``good'' hidden images with high confidence on both original and rescaled data.
Essentially, the prediction confidence level of the hidden images will be no better than the prediction confidence of the original images, after rescaling has been applied.


\begin{figure}[H]
    \centering
    \includegraphics[width=0.3\textwidth]{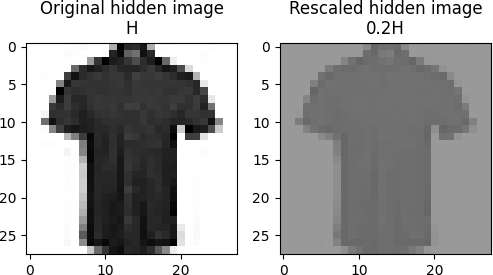}
    \caption{Original and rescaled images.}
    \label{fig:stego-confi-ex}
\end{figure}

To continue this investigation, we also perform experiments on the EMNIST Balanced dataset (Figure \ref{fig:emnist}) and CIFAR-10 dataset (Figure \ref{fig:cifar}). 
The figures suggest that the null space-based image steganography method is also applicable to more complicated and colorful images. 
To guarantee the capability of concealing an image within any chosen cover image, we can train a neural network model with a large null space.

\begin{figure}[H]
    \centering
    \begin{subfigure}{0.22\textwidth}
        \includegraphics[width=\linewidth]{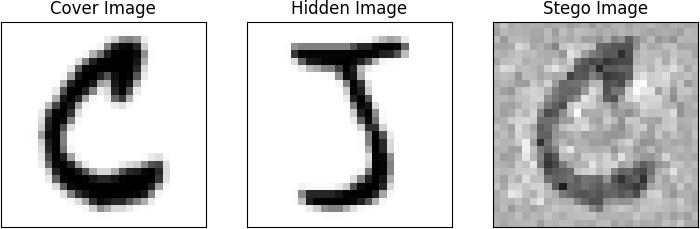}
    \end{subfigure}
    \hfill
    \begin{subfigure}{0.22\textwidth}
        \includegraphics[width=\linewidth]{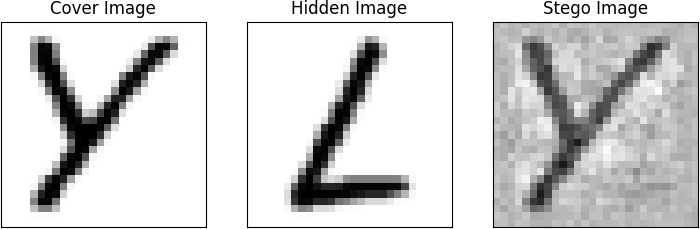}
    \end{subfigure}
    \hfill
    \begin{subfigure}{0.22\textwidth}
        \includegraphics[width=\linewidth]{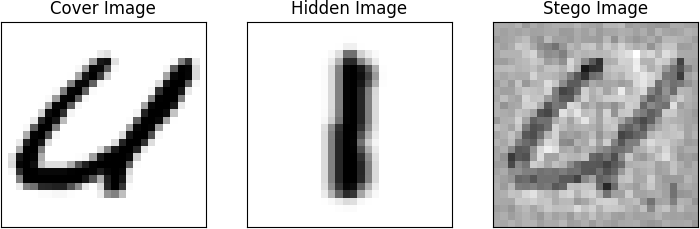}
    \end{subfigure}
    \hfill
    \begin{subfigure}{0.22\textwidth}
        \includegraphics[width=\linewidth]{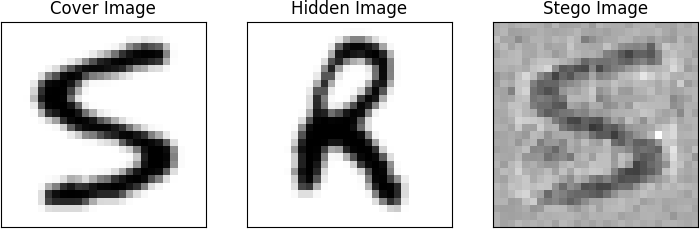}
    \end{subfigure}
    \caption{More examples on EMNIST dataset.}
    \label{fig:emnist}
\end{figure}

\begin{figure}[H]
    \begin{subfigure}{0.22\textwidth}
        \includegraphics[width=\linewidth]{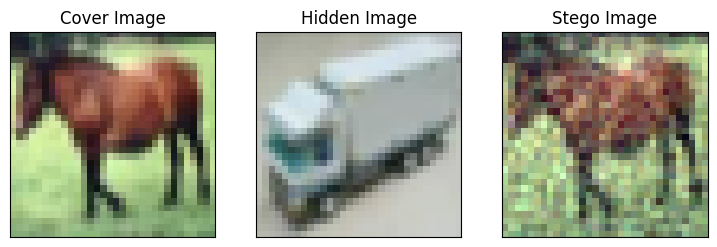}
    \end{subfigure}
    \hfill
    \begin{subfigure}{0.22\textwidth}
        \includegraphics[width=\linewidth]{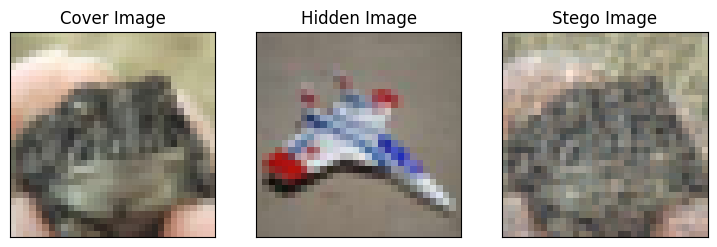}
    \end{subfigure}
    \hfill
    \begin{subfigure}{0.22\textwidth}
        \includegraphics[width=\linewidth]{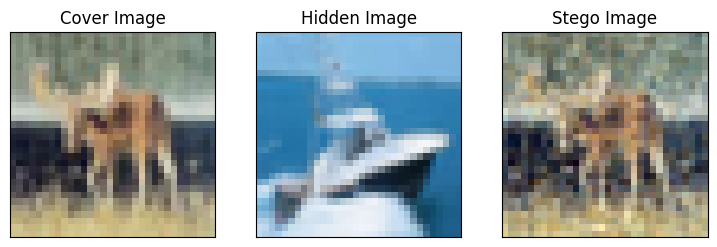}
    \end{subfigure}
    \hfill
    \begin{subfigure}{0.22\textwidth}
        \includegraphics[width=\linewidth]{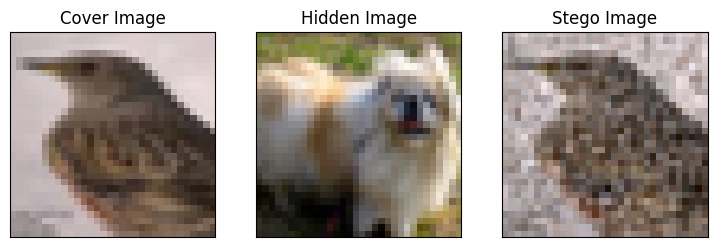}
    \end{subfigure}
    \caption{More examples on CIFAR-10 dataset.}
    \label{fig:cifar}
\end{figure}

\section{Discussions and conclusions}\label{sect:conclusion}
\paragraph{What we see and what NN sees:} 
From the analysis and experimental results presented above, it is evident that NNs do not perceive visual information in the same way as humans. Figure \ref{fig:compare-human-nn} compares what we see and what the NN ``sees''.
After removing null space components, the remaining crucial parts for prediction have fewer visual patterns than their initial appearance.
Except for that, distinct neural networks would see the same image differently, even when they have null spaces of the same dimensions, as illustrated in Figure \ref{fig:nnsees2}.

\begin{figure}[H]
\centering
\begin{subfigure}{\textwidth}
  \centering
  \includegraphics[width=.98\linewidth]{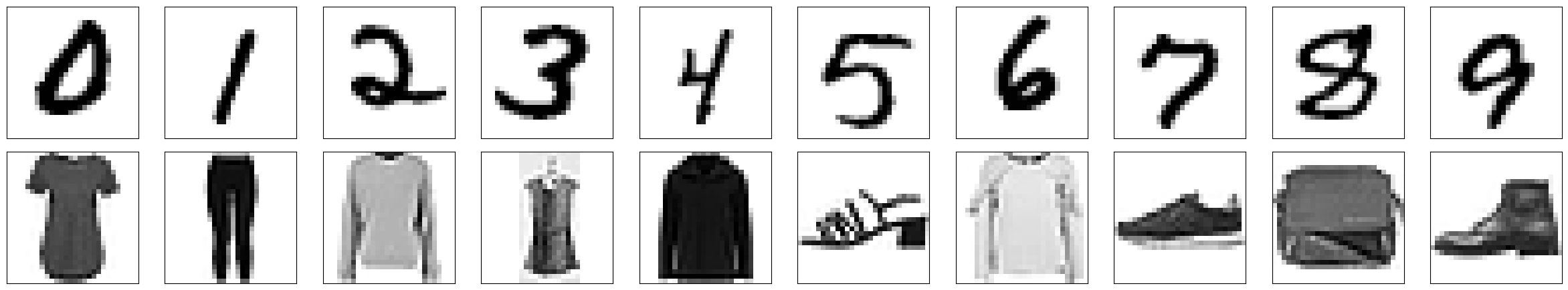}
  \caption{What we see.}
  \label{fig:wesee}
\end{subfigure}%
\hfill
\begin{subfigure}{\textwidth}
  \centering
  \includegraphics[width=.98\linewidth]{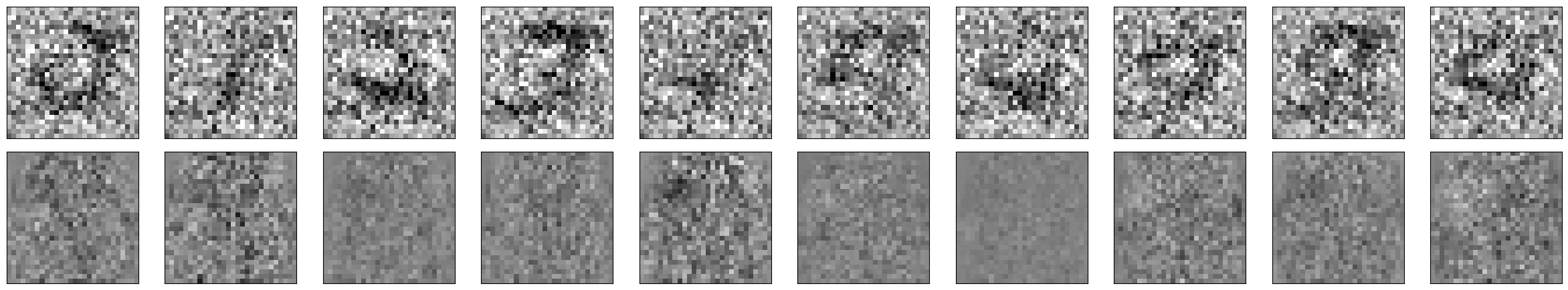}
  \caption{What NN sees.}
  \label{fig:nnsees}
\end{subfigure}
\caption{Comparison of what is visualized by humans and NNs.}
\label{fig:compare-human-nn}
\end{figure}

\begin{figure}[H]
    \centering
    \includegraphics[width=.98\linewidth]{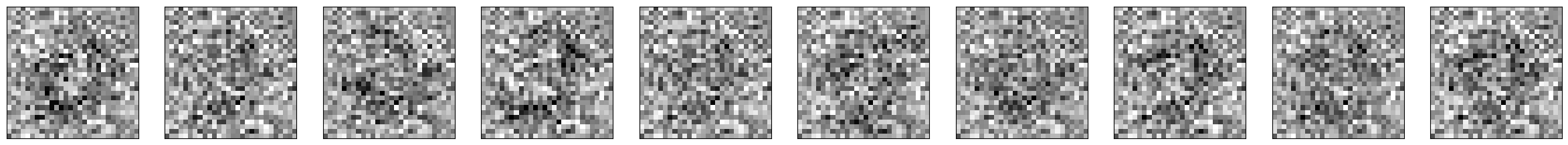}
    \caption{What another NN sees given the same set of images in Figure \ref{fig:compare-human-nn}.}
    \label{fig:nnsees2}
\end{figure}

\paragraph{Null space of NNs and reliability issue:} Ball pointed out that there are reproducibility and reliability issues with AI \cite{ball2023ai}. 
Previously, studies have found that neural network models could pick up on irrelevant features to succeed in the classification task,  raising questions about their reliability. 
In this paper, the null space analysis on neural networks provides additional insights into these reliability concerns, 
and shows that the NN may pick up on hidden features that are not just irrelevant, but crafted to confuse the NN. 
In addition to being too closely focused on aligning to the particular patterns in the training data, the design of neural network architectures may also raise risks that cannot be solved by merely increasing the dataset size. 
Importantly, it is the architecture that leads to the null space weaknesses shown in this paper.

In addition, stego images created by null space methods cannot be used to improve the training. 
As described in \cite{akhtar2018threat}, the adversarial example/image is a modified version of a clean image that is intentionally perturbed to mislead machine learning models, such as deep neural networks. 
Some studies have shown that it may be possible to harden NN against adversarial attacks. 
Further, it was observed by Szegedy et al.\cite{szegedy2014intriguing} that the robustness of deep neural networks against adversarial examples could be improved by adversarial training, where the idea is to include adversarial examples in the training data. 
It is crucial to note that while adversarial examples using previous techniques may improve the training, they cannot solve the problem caused by the null space of a neural network. 

The null space analysis of NNs is not limited to the NNs in image classification tasks. In this study, we select image steganography as an application to better visualize the impact of the null space.
The existence of the null space, inherent in the neural network's architecture, implies that one can always use the null space vectors to fool a neural network or the user of a neural network, at least when the image projected onto the null space is close enough to the original to fool a human viewer. 

\normalem
\printbibliography

\section*{Appendices}
\begin{appendices}

\section{Proofs and additional results from Section \ref{sect:nullspace}}\label{appNS}
\subsection{Properties of the null space in nonlinear maps}
\begin{prop}\label{prop:n-subsp}
Let $f: \mathbb{R}^n\rightarrow \mathbb{R}^m$ be a nonlinear map.
Then the set
$N(f):=\{ \vec v\in \mathbb{R}^n: f(\vec x) = f(\vec x+a\vec v) \text{ for all } \vec x\in \mathbb{R}^n \text{ and } a\in \mathbb{R}\}$
is a subspace of $\mathbb{R}^n$.
\end{prop}
\begin{proof}
First, by definition, $\vec 0\in N(f)$, and for every $\vec v\in N(f)$ and for every $c\in\mathbb{R}$, $c\vec v\in N(f)$.
Second, when $\vec v_1,\vec v_2\in N(f)$,
for every $\vec x\in \mathbb{R}^n$ and $a\in \mathbb{R}$ we have $f(\vec x+a(\vec v_1+\vec v_2)) = f(\vec x+a\vec v_1+a\vec v_2)=  f(\vec x+a\vec v_1)=  f(\vec x) $,
hence $\vec v_1+\vec v_2\in N(f)$.
Therefore $N(f)$ is a subspace of $\mathbb{R}^n$.
\end{proof}

\begin{prop}\label{prop:same-image}
    Given a nonlinear map $f: \mathbb{R}^n\rightarrow \mathbb{R}^m$. For any $\vec{x}\in N(f)$ and $\vec{y}\in N(f)$, $f(\vec{x})=f(\vec{y})=f(\vec{0})$.
\end{prop}
\begin{proof}
For any $\vec{x}, \vec{y} \in N(f)$, $f(\vec{x})=f(\vec{x}+\vec{0})=f(\vec{0})=f(\vec{y}+\vec{0})=f(\vec{y})$.
\end{proof}

Given a nonlinear map $f: \mathbb{R}^n\rightarrow \mathbb{R}^m$,
instead of its null space $N(f)$,
we also considered its null set, 
$\mathcal{N}(f):=\{\vec{v}\in\mathbb{R}^n: f(\vec{x})=f(\vec{x}+\vec{v}) \text{ for all } \vec{x}\in\mathbb{R}^n\}$.
It is easy to see that $N(f)$ is a subset of $\mathcal{N}(f)$,
and $\mathcal{N}(f)$ may not be a vector space in general.
However,
$\mathcal{N}(f)$ also has some structure, any vector in $\mathcal{N}(f)$ that is not in the subspace $N(f)$ is part of a set of integer periodic null elements.

\begin{prop}
    Given a nonlinear map $f: \mathbb{R}^n\rightarrow \mathbb{R}^m$, consider the set $\mathcal{N}=\{\vec{v}\in\mathbb{R}^n:f(\vec{x})=f(\vec{x}+\vec{v}) \text{ for all } \vec{x}\in\mathbb{R}^n\}$. 
        If $\vec{v}\in \mathcal{N}$ but $\vec{v} \notin N(f)$, then for any integer $k$, $k\vec{v}\in \mathcal{N}$. 
\end{prop}

\begin{proof}
Let $\vec{v_1}, \vec{v_2}\in \mathcal{N}(f) $,
we prove that $-\vec{v_1}$ and $\vec{v_1}+ \vec{v_2}\in \mathcal{N}(f) $.
Indeed,
$\vec{v_1}\in \mathcal{N}(f) $ implies that $f(\vec{x}-\vec{v_1})=f(\vec{x}-\vec{v_1}+\vec{v_1})= f(\vec{x})$ 
for every $\vec{x}\in\mathbb{R}^n$,
hence $-\vec{v_1} \in \mathcal{N}(f) $.
And, $\vec{v_1}, \vec{v_2}\in \mathcal{N}(f) $ implies that $f(\vec{x}) =  f(\vec{x}+\vec{v_1})= f(\vec{x}+\vec{v_1}+\vec{v_2})$ for every $\vec{x}\in\mathbb{R}^n$ ,
hence $\vec{v_1}+ \vec{v_2}\in \mathcal{N}(f) $.
Therefore, if $\vec{v}\in \mathcal{N}(f)$, $k\vec{v}\in \mathcal{N}$ for any integer $k$.
\end{proof}

\subsection{Proofs of lemmas from Section \ref{sect:nullspace}}

\begin{proof}[\textbf{Proof of Lemma \ref{lem:pn_subspace}}]
Assume the nonlinear map $f: \mathbb{R}^n\rightarrow \mathbb{R}^m$ has a decomposition $f=f_2\circ f_1$ where $f_1$ is linear. 
For any $\vec{x}_{null}\in\text{Null}(f_1)$, $f_1(\vec{x})=f_1(\vec{x}+\vec{x}_{null})$ holds for every $\vec{x}\in\mathbb{R}^n$, 
$f(\vec{x}+\vec{x}_{null}) = f_2\circ f_1(\vec{x}+\vec{x}_{null}) = f_2\circ f_1(\vec{x})= f(\vec{x})$.
Therefore, $\vec{x}_{null}\in N(f)$ and $PN_{f_1}=\text{Null}(f_1)$ is a subspace of $N(f)$ with $\dim PN(f)\leq \dim N(f)$.
\end{proof}

\begin{proof}[\textbf{Proof of Lemma \ref{lem:pn}.}]
Given the nonlinear map $f: \mathbb{R}^n\rightarrow \mathbb{R}^m$, consider the quotient space
$\mathbb{R}^n/ N(f) := \{[\vec v]: \vec v\in \mathbb{R}^n\}$, 
where $[\vec v]=\{\vec v+ \vec y:\vec y\in N(f)\} $.
The addition $[\vec v]+ [\vec u] = [\vec v+\vec u]$, scalar multiplication $\lambda [\vec v] = [\lambda \vec v] $, and $[\vec 0]=N(f)\in \mathbb{R}^n/ N(f)$ make $\mathbb{R}^n/ N(f)$ a vector space (of dimension $n-\dim(N(f))$).
Denote $f_1$ to be the quotient map from $\mathbb{R}^n$ to $\mathbb{R}^n/ N(f)$ (i.e., $f_1(\vec v) = [\vec v]$).
Then, $f_1$ is linear and $ \text{Null}(f_1) = N(f)$.
Define $f_2: \mathbb{R}^n/ N(f)\rightarrow \mathbb{R}^m$ to be the map: $f_2([\vec v])= f(\vec v)$.
To show $f_2$ is well-defined,
notice that when $[\vec v]=[\vec u] $, $\vec u=\vec v+\vec y$ for some $\vec y\in N(f)$,
hence $f(\vec u)=f(\vec v+\vec y) = f(\vec v) $ by the definition of $N(f)$.
Now $f= f_2\circ f_1$ is the decomposition needed.
\end{proof}

\section{Related results in section \ref{CNN}}\label{appCNN}
The convolution operation is linear. 
The convolution of an image and a kernel can also be written in a matrix-vector multiplication form. 
For example, consider the case of a  $5\times5\times1$ image convolved with a $3\times3\times1$ kernel.
Figure \ref{fig:conv-ex} and \ref{fig:valid-pad} show the convolution with valid padding. A $5\times5\times1$ image convolved with a $3\times3\times1$ kernel with valid padding can be written as a $9\times25$ matrix (kernel) multiplied with a $25\times1$ vector (image). 

The convolution with the same padding keeps the input dimension and the output dimension the same by appending zero values in the outer frame of the images.
Figure \ref{fig:same-pad} shows the matrix-vector multiplication form of a convolution operation with the same padding, 
that is, a $25\times25$ matrix (kernel) multiplied with a $25\times 1$ vector. 

\begin{figure}[H]
    \centering
    \includegraphics[width=0.5\textwidth]{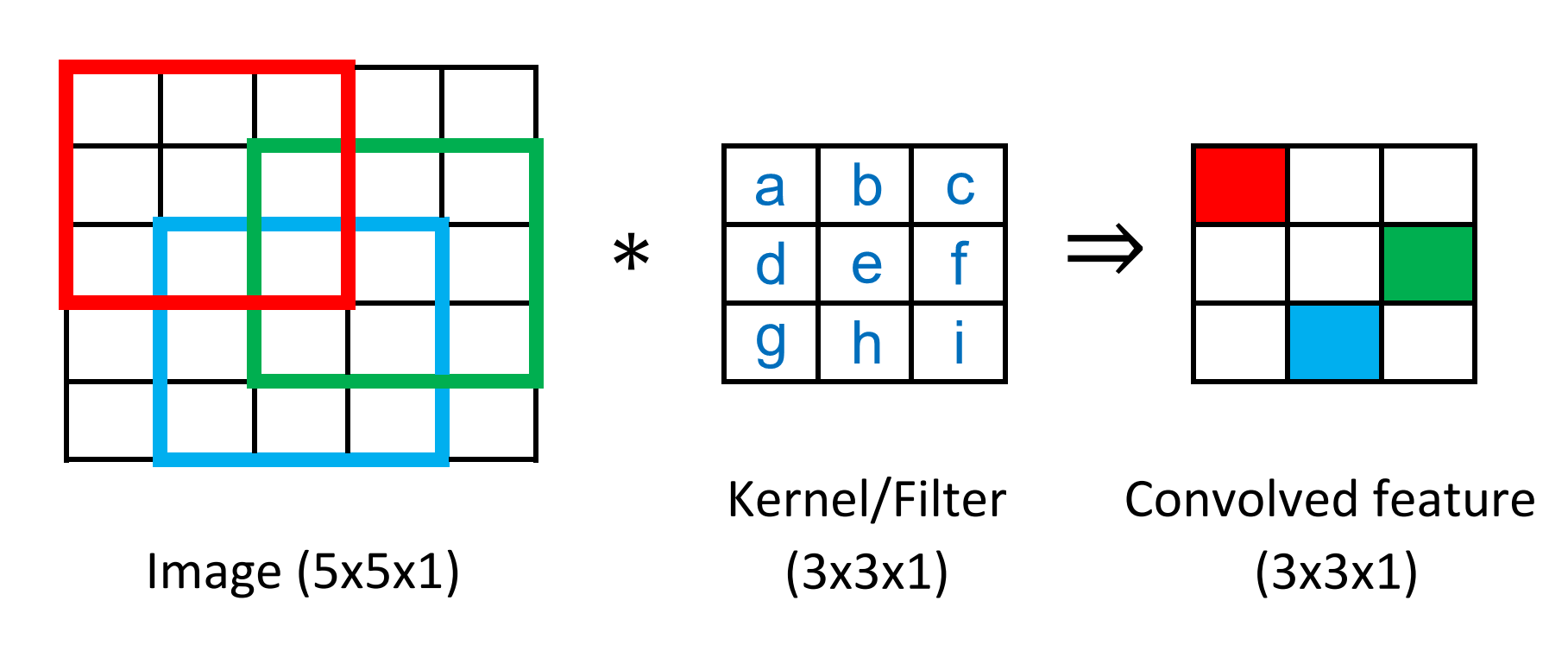}
    \caption{Convolution example}
    \label{fig:conv-ex}
\end{figure}

\begin{figure}[H]
\centering
\begin{subfigure}{.5\textwidth}
  \centering
  \includegraphics[width=.8\linewidth]{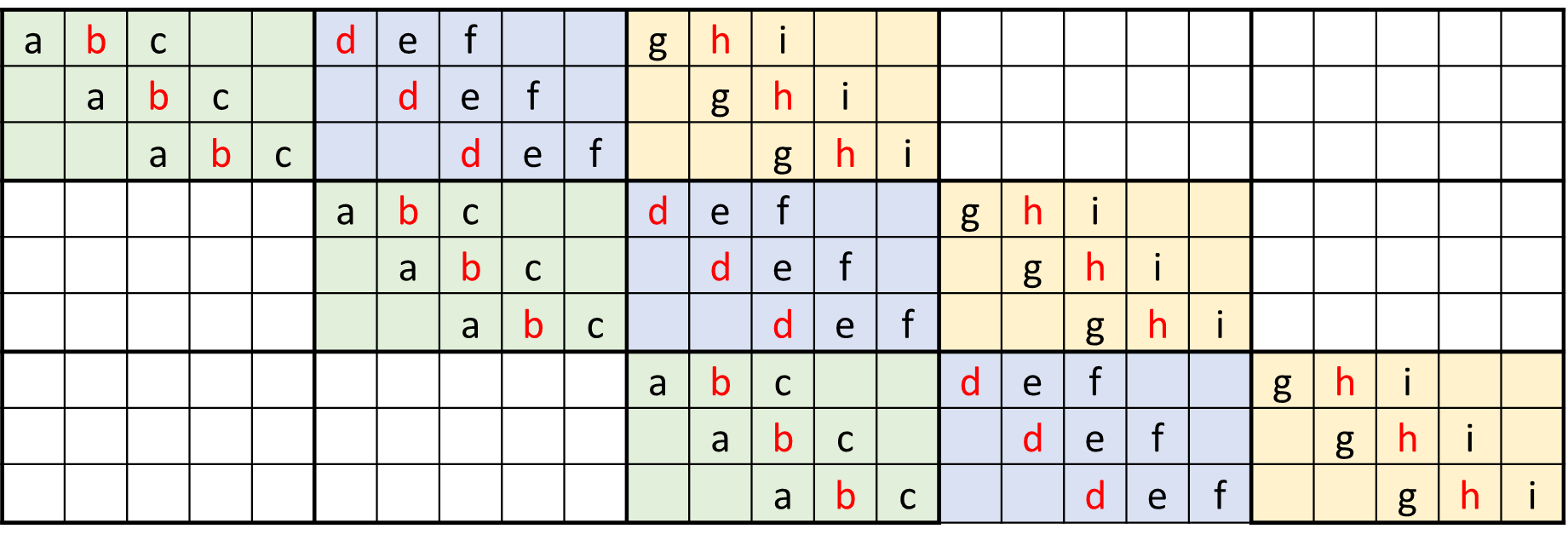}
  \caption{Convolution with valid padding}
  \label{fig:valid-pad}
\end{subfigure}%
\begin{subfigure}{.5\textwidth}
  \centering
  \includegraphics[width=.8\linewidth]{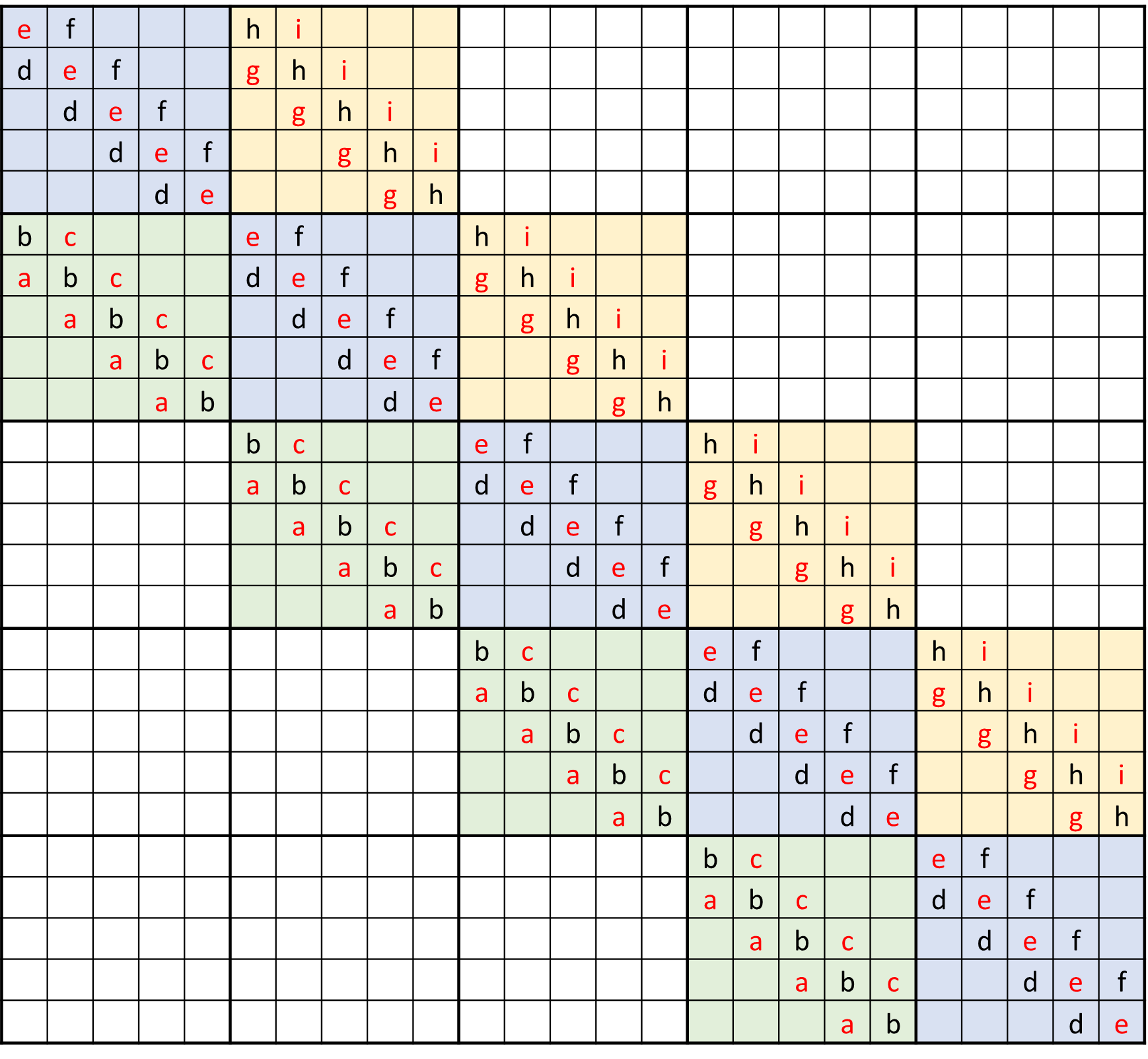}
  \caption{Convolution with same padding}
  \label{fig:same-pad}
\end{subfigure}
\caption{Standard matrices for convolution operation.}
\label{fig:conv_mat}
\end{figure}

We can prove that for almost all cases, the convolution operations with valid padding or same padding are full rank, i.e., the dimension of the null space is exactly the difference of the dimensions of input and output.
More precisely, consider all the convolution operations convolving $n_1\times n_2$ inputs with $k_1\times k_2$ kernels. 
The set of all $k_1\times k_2$ kernels can be identified with $\mathbb{R}^{k_1\times k_2}$, and we may equip the Lebesgue measure on the set.
The following Lemma shows that, for all but a zero-measure subset of all kernels, the convolution operation has full rank.

\begin{lemma}\label{lem:conv}
Let $ n_1,n_2,k_1,k_2$ be positive integers and $n_1\geq k_1$, $n_2\geq k_2$.
For almost all $k_1\times k_2$ kernels,
the (valid padding or same padding) convolution operation of the kernel acting on $n_1\times n_2$ images has full rank.
\end{lemma}

\begin{proof}
Let $M$ be the matrix representation of the convolution operation (see Figure \ref{fig:conv_mat} for both valid padding and same padding cases.) 

To begin with, let's focus on the convolution with valid padding (see Figure \ref{fig:valid-pad}.)
The null space of a convolution operation is $\text{Null}(M)$.
By switching columns in the matrix $M$, we may assume that there is a weight $a$ of the kernel that appears and only appears on every entry of the diagonal of $M$.
As an example, for the matrix in Figure \ref{fig:valid-pad}, we switch the columns with the new order $(1, 2, 3, 6, 7, 8, 11, 12, 13, 4, 5, 9, 10, 14, 15, \cdots)$.
Then, the diagonal has entries all equal to $a$.
Let $M_1$ be the matrix consisting of the first $\min(k_1,k_2)$ columns and the first $\min(k_1,k_2)$ rows of $M$, which is a square matrix with diagonal entries all equal to $a$.
In the following, we show $M_1$ has full rank for all but a finite number of choices of $a$, hence $M$ also has full rank for almost all kernels.
We have that $M_1$ is a square matrix, and $M_1=aI+ B$, where $I$ is the identity matrix and $B$ is a square matrix independent of $a$ with diagonal entries all equal to zero.
Therefore, $M_1$ is invertible (i.e., has full rank) if and only if $-a$ is not an eigenvalue of $B$.
Fix the other values in the kernel, only change the value of $a$, then the matrix $B$ is fixed.
While $B$ has finitely many eigenvalues, for every $B$, for all but a finite number of choices of $a$, $M_1$ has full rank. 
Thus, for all but a zero-measure subset of all $k_1\times k_2$ kernels, $M_1$ and $M$ has full rank.

For the convolution with same padding, let $M$ be the matrix representation of a given kernel (see Figure \ref{fig:same-pad}.) The diagonal entries are all equal.
Following the proof above, we know that $M$ is full rank for almost all kernels.
\end{proof}

\section{A brief introduction to fully connected neural networks}\label{appIntro}
In order to establish notation and to ensure a consistent presentation, the following section describes the NN architecture considered here. 
As illustrated in Figure \ref{fig:fcnn}, we will consider a fully connected neural network $f: \mathbb{R}^{n_0}\rightarrow\mathbb{R}^{n_{K+1}}$ with $n_0$-dimensional input and $n_{K+1}$-dimensional output. 
Assume that there are $K$ hidden layers, each having $n_i$, for nodes $i=1,\cdots, K$. 
The weights from the $i^{th}$ layer, $i=1, \cdots, K+1$ are denoted by weight matrix $W_i$ and bias vector $\vec{b}_i$, respectively. 
More precisely, the $(m_1, m_2)$ entry in weight matrix $W_i$ connects the $m_2^{th}$ output in $(i-1)^{th}$ layer and the $m_1^{th}$ node in the $i^{th}$ layer; the $m^{th}$ entry in vector $\vec{b}_i$ is the bias term of the $m^{th}$ node in the $i^{th}$ layer.
The activation functions are denoted by $\sigma$. 

\begin{figure}[H]
    \centering
    \includegraphics[width=.8\textwidth]{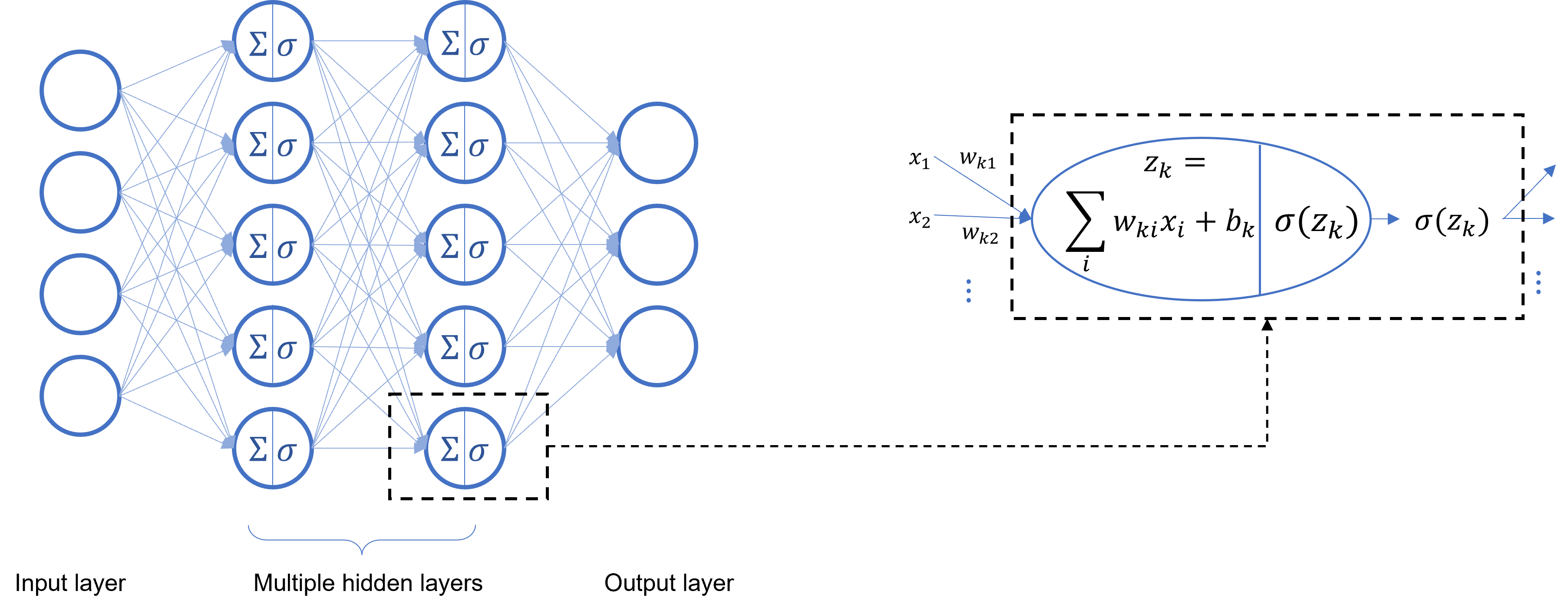}
    \caption{Architecture of FCNN}
    \label{fig:fcnn}
\end{figure}

As shown in Figure \ref{fig:fcnn}, assume the input values are $\{x_i\}$, output of the $k^{th}$ neuron in the $j^{th}$ layer is given by 
$$\sigma(z_k^{[j]})=\sum\limits_{i}W_j(k, i)x_i+\vec{b}_j(k) ,$$ 
and outputs of the $j^{th}$ layer can be written in a compact matrix form
$$\sigma(\vec{z}^{[j]})=W_j\vec{x}+\vec{b}_j.$$
The function represented by this neural network is 
$$f(\vec{x})=W_{K+1}\;\sigma(W_{K}\;\sigma(\cdots W_2\;\sigma(W_1\vec{x}+\vec{b}_1)+\vec{b}_2\cdots)+\vec{b}_K)+\vec{b}_{K+1}.$$

Define a sequence of linear and affine transformations as $T_i(\vec{x})=W_i \vec{x}, A_i(\vec{x}) = \vec{x} +\vec{b}_i$ for $i=1,\dots,k+1$, respectively. 
The function of FCNN $f$ can also be represented by a composition of maps:
$$f = A_{k+1} \circ T_{k+1} \circ \sigma \circ T_k \circ \cdots \circ A_2 \circ T_2 \circ \sigma \circ A_1 \circ T_1.$$

The notations established here are used throughout the paper.

\end{appendices}

\end{document}